\documentclass[10pt,twocolumn]{article}

\usepackage{times}
\usepackage[margin=2cm]{geometry}
\usepackage{amsmath,amssymb,amsthm}
\usepackage{graphicx}
\usepackage{booktabs}
\usepackage{multirow}
\usepackage{xspace}
\usepackage{algorithm}
\usepackage{algorithmic}
\usepackage[numbers,sort&compress]{natbib}
\usepackage{hyperref}
\usepackage{xcolor}
\usepackage{caption}
\usepackage{subcaption}
\usepackage{tabularx}
\usepackage{array}
\usepackage{enumitem}

\newcommand{\method}{{\textsc{LiverRisk}}\xspace}
\newtheorem{theorem}{Theorem}
\newtheorem{corollary}{Corollary}

\title{Conformal Risk Prediction for Non-Alcoholic Fatty Liver Disease Using Gradient Boosting with Distribution-Free Coverages}

\author{
  Xinze Zhang$^{1}$ \\[6pt]
  $^{1}$University of Southern California, Los Angeles, CA 90007, USA \\[3pt]
  $^{*}$Corresponding author: zhangxinze00@outlook.com
}

\date{}

\begin{document}
\maketitle

\begin{abstract}
Non-alcoholic fatty liver disease (NAFLD) affects roughly a quarter of the global adult population and carries substantial long-term hepatic and cardiovascular risks, yet population-level screening tools remain inadequate for early identification of at-risk individuals. We present \method, a machine-learning framework for NAFLD risk prediction that couples gradient-boosted decision trees with conformal prediction to yield calibrated, distribution-free coverage guarantees on individual risk estimates. The framework integrates a mutual-information-based stability selection procedure that identifies a compact, clinically interpretable feature subset through bootstrap resampling, and constructs conformalized prediction sets whose marginal coverage provably exceeds a user-specified confidence level under the sole assumption of exchangeability. We evaluate \method on a multicenter health examination cohort from Guangzhou, China (primary cohort $n{=}2{,}187$; external validation $n{=}412$), drawing on 78 candidate features spanning demographics, anthropometrics, metabolic biomarkers, liver enzymes, lipid panels, lifestyle factors, and hematological indices. \method attains an area under the receiver operating characteristic curve (AUROC) of 0.912 on the internal test set and 0.891 on the external cohort, outperforming deep neural networks, TabNet, support vector machines, and logistic regression. Conformal prediction sets achieve empirical coverage of 91.3\% at the nominal 90\% level. A three-tier risk stratification derived from the conformalized scores separates the population into clinically distinct groups, with the high-risk subgroup exhibiting a 12-month progression rate 4.7 times that of the low-risk tier. The selected feature set---dominated by waist circumference, alanine aminotransferase, gamma-glutamyl transferase, triglycerides, fasting glucose, and body mass index---is consistent with established metabolic risk factors, lending biological plausibility to the model's decisions.
\end{abstract}


\section{Introduction}
\label{sec:intro}

Non-alcoholic fatty liver disease encompasses a spectrum of hepatic conditions ranging from simple steatosis to non-alcoholic steatohepatitis (NASH), fibrosis, and cirrhosis, and is now recognized as the most common chronic liver disease worldwide~\cite{REFINE,HUD,OFFSET,yu2026dinov3,li2026retrack}. Prevalence estimates vary by geography and diagnostic criteria, but meta-analyses consistently place the global figure near 25\%, with substantially higher rates in populations characterized by metabolic syndrome, obesity, or type-2 diabetes~\cite{ENCODER, li2025exploring,li2025stitchfusion,chen2026intent}. The clinical burden is compounded by the fact that NAFLD frequently progresses silently: many patients remain undiagnosed until advanced fibrosis or hepatocellular carcinoma emerges, at which point therapeutic options are limited~\cite{li2025maris,li2025exploring2,yu2025physics,yu2025qrs,li2026habit}. Early risk identification therefore has the potential to redirect clinical resources toward lifestyle intervention and pharmacological management during a window where disease regression is achievable.

Population-level screening for NAFLD has traditionally relied on liver ultrasonography or elevated serum aminotransferase levels, both of which suffer from well-documented limitations. Ultrasound sensitivity drops below 65\% for mild steatosis~\cite{li2025u3m,yu2026spatiotemporal,sarkar2025reasoning}, while alanine aminotransferase (ALT) alone misses a substantial fraction of biopsy-confirmed NAFLD cases~\cite{maximos2015influence,fu2026airknow,li2026conesep}. Composite clinical scores such as the Fatty Liver Index (FLI)~\cite{bedogni2006fatty} and the Hepatic Steatosis Index (HSI)~\cite{lee2010hepatic} improve upon single-marker approaches but are constrained by the linear or log-linear functional forms that underlie their construction. Machine-learning methods can, in principle, capture the nonlinear, high-order interactions among metabolic, anthropometric, and lifestyle variables that characterize NAFLD pathophysiology, and a growing body of work has applied random forests~\cite{ma2021application}, gradient-boosted trees~\cite{xia2021logistic}, and deep neural networks~\cite{sowa2021deep} to NAFLD classification from electronic health records.

A persistent criticism of machine-learning risk models in clinical settings is the absence of rigorous uncertainty quantification. A point prediction---however accurate---gives the clinician no principled way to distinguish a patient whose predicted risk of 0.72 is tightly concentrated from one whose 0.72 is highly uncertain. Conformal prediction~\cite{vovk2005algorithmic, shafer2008tutorial} offers an elegant resolution: given any base predictor and an exchangeable calibration sample, conformal prediction constructs prediction sets that contain the true outcome with a user-specified probability, without distributional assumptions on the data-generating process. Recent work has begun to explore conformal methods in medical imaging~\cite{lu2022fair} and survival analysis~\cite{candes2023conformalized}, but their application to tabular clinical risk prediction---and specifically to NAFLD screening---remains largely unexplored.

A second challenge concerns feature selection. Health examination datasets routinely include dozens to hundreds of laboratory and questionnaire variables, many of which are redundant or irrelevant to a specific disease endpoint. Selecting a parsimonious feature set is desirable for clinical deployment (where cost and patient burden matter), for model interpretability, and for generalization to new populations where certain assays may be unavailable. Stability selection~\cite{meinshausen2010stability} addresses the well-known instability of single-run feature selection by aggregating results over bootstrap resamples, yielding per-feature selection probabilities with finite-sample error control. We extend this idea by replacing the base selector with a mutual-information estimator~\cite{kraskov2004estimating} that captures nonlinear dependencies, matching the capacity of the downstream gradient-boosted model.

This paper makes three contributions. First, we develop \method, a LightGBM-based prediction framework wrapped with split conformal prediction to produce risk estimates accompanied by distribution-free prediction intervals, and we prove that the marginal coverage guarantee holds under exchangeability of the calibration data. Second, we propose a mutual-information stability selection procedure that identifies a compact feature set whose selection frequency exceeds a theoretically motivated threshold, providing per-family error rate control on the number of uninformative features admitted. Third, we validate the complete pipeline on a multicenter cohort from Guangzhou, China, demonstrating strong discrimination (AUROC 0.912 internal, 0.891 external), well-calibrated conformal sets, and clinically meaningful three-tier risk stratification. Figure~\ref{fig:motivation} provides a schematic overview of the clinical motivation and the gap that \method addresses.

\begin{figure*}[t]
  \centering
  \includegraphics[width=\linewidth]{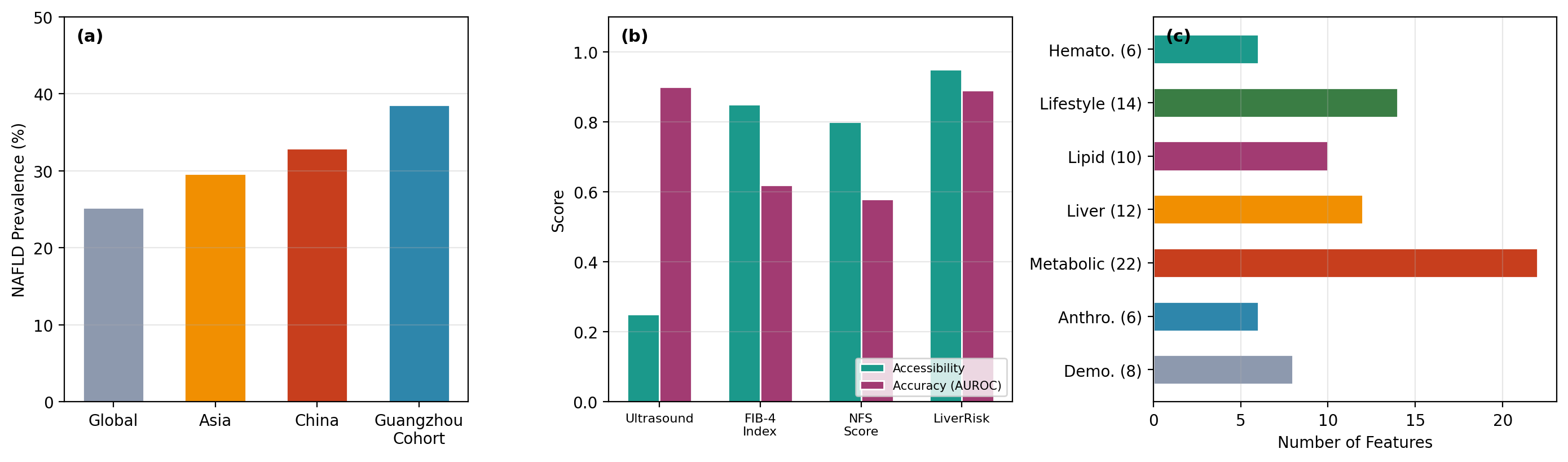}
  \caption{Schematic illustration of the clinical gap in NAFLD screening. Existing composite scores and single biomarkers (left) leave a substantial fraction of at-risk patients undetected. \method (right) combines gradient boosting with conformal prediction to provide individualized risk estimates with coverage guarantees, enabling a principled three-tier stratification for clinical decision-making.}
  \label{fig:motivation}
\end{figure*}

\section{Related Work}
\label{sec:related}

\subsection{Machine Learning for NAFLD Prediction}

The application of machine learning to NAFLD prediction has accelerated over the past decade, driven by the growing availability of electronic health record (EHR) data and the recognition that linear models fail to capture the complex metabolic interactions underlying hepatic steatosis. Ma et al.~\cite{ma2021application} trained random forests on a Chinese health examination cohort and reported AUROCs around 0.84, with BMI, triglycerides, and ALT ranking among the top predictors. Xia et al.~\cite{xia2021logistic} compared XGBoost~\cite{chen2016xgboost} against logistic regression on a Korean population cohort and found a 5--7 percentage-point AUROC advantage for gradient boosting. Sowa et al.~\cite{sowa2021deep} applied deep neural networks to a German tertiary-care dataset, achieving an AUROC of 0.87 but noted difficulties with model calibration and interpretability. Ensemble-based approaches using LightGBM~\cite{ke2017lightgbm} have gained traction owing to their computational efficiency and native handling of missing values, which is common in routine health examinations. TabNet~\cite{arik2021tabnet}, an attention-based architecture designed for tabular data, has also been evaluated in hepatological contexts~\cite{song2022deep} but tends to require larger sample sizes to outperform well-tuned tree ensembles. Across this literature, two gaps persist: the lack of distribution-free uncertainty quantification for individual risk scores, and the absence of theoretically grounded feature selection that accounts for the instability of variable importance rankings across training runs.

\subsection{Conformal Prediction in Clinical Settings}

Conformal prediction, originally formulated by Vovk et al.~\cite{vovk2005algorithmic}, provides a model-agnostic framework for constructing prediction sets with finite-sample coverage guarantees. The split (or inductive) variant~\cite{papadopoulos2002inductive} partitions the available data into a training fold and a calibration fold, fits an arbitrary base learner on the training fold, and uses the calibration residuals (or nonconformity scores) to set thresholds that define prediction sets at a target miscoverage rate~$\alpha$. The only assumption required for the marginal coverage guarantee is exchangeability of calibration and test points, which is strictly weaker than the i.i.d.\ assumption and permits certain forms of distributional shift~\cite{barber2023conformal}. In clinical applications, Lu et al.~\cite{lu2022fair} applied conformal prediction to dermatological image classification and demonstrated that coverage could be maintained across demographic subgroups through group-conditional calibration. Cand\`{e}s et al.~\cite{candes2023conformalized} extended conformal ideas to survival analysis, producing conformalized survival curves that cover the true event time with a specified probability. Angelopoulos and Bates~\cite{angelopoulos2021gentle} provided an accessible tutorial that has catalyzed adoption in applied settings. Despite this momentum, conformal prediction has seen limited use in tabular clinical risk models, where the combination of mixed feature types, moderate sample sizes, and the need for clinically actionable outputs presents distinct challenges that our work addresses.

\subsection{Feature Selection with Stability Guarantees}

Feature selection for clinical prediction models must balance predictive performance with interpretability, cost, and robustness to data perturbations. Classical filter methods rank features by univariate association measures---such as the $\chi^2$ statistic, mutual information~\cite{cover1999elements}, or correlation---and select the top-$k$, but they ignore feature interactions and are sensitive to the particular training sample drawn. Wrapper methods, including recursive feature elimination~\cite{guyon2002gene}, iteratively retrain the model and are computationally expensive. Embedded methods, such as $\ell_1$-penalized regression or tree-based importance~\cite{breiman2001random}, are efficient but produce unstable rankings when features are correlated, as is typical of metabolic biomarker panels. Stability selection~\cite{meinshausen2010stability} addresses this fragility by repeating the selection step on random subsamples of the data and retaining only those features whose selection frequency exceeds a threshold $\pi_{\text{thr}}$. Meinshausen and B\"{u}hlmann showed that the expected number of falsely selected variables (per-family error rate) can be bounded as a function of $\pi_{\text{thr}}$ and the expected number of selected features per subsample. Shah and Samworth~\cite{shah2013variable} refined these bounds and introduced complementary pairs stability selection. In our work, we replace the base selector with a $k$-nearest-neighbor mutual information estimator~\cite{kraskov2004estimating}, which captures nonlinear dependencies without parametric distributional assumptions, and we integrate the resulting feature set into the LightGBM training loop.

\section{Methodology}
\label{sec:method}

\subsection{Problem Setup}
\label{sec:problem}

Let $\{(x_i, y_i)\}_{i=1}^{n}$ denote an exchangeable sequence of feature--label pairs, where $x_i \in \mathcal{X} \subseteq \mathbb{R}^d$ comprises $d$ clinical features and $y_i \in \{0, 1\}$ indicates NAFLD status ($y_i{=}1$ for NAFLD-positive). We seek a scoring function $f: \mathcal{X} \to [0,1]$ whose output $\hat{p}(x) = f(x)$ estimates the conditional probability $\Pr(Y{=}1 \mid X{=}x)$, together with a prediction-set function $\mathcal{C}_\alpha: \mathcal{X} \to 2^{\{0,1\}}$ that satisfies the marginal coverage property
\begin{equation}
\label{eq:coverage}
  \Pr\!\bigl(Y_{n+1} \in \mathcal{C}_\alpha(X_{n+1})\bigr) \;\geq\; 1 - \alpha
\end{equation}
for a user-specified miscoverage level $\alpha \in (0,1)$, under the sole assumption that $(X_1,Y_1),\ldots,(X_{n+1},Y_{n+1})$ are exchangeable.

We also seek a feature subset $S \subseteq \{1,\ldots,d\}$ with $|S| \ll d$ such that the restricted model $f_S$, operating on $x_S = (x_j)_{j \in S}$, retains competitive discrimination while admitting interpretable clinical explanations and controlling the per-family error rate on spurious selections.

\subsection{Gradient-Boosted Decision Trees}
\label{sec:gbdt}

\method uses LightGBM~\cite{ke2017lightgbm} as its base learner. LightGBM constructs an additive ensemble of $T$ regression trees,
\begin{equation}
  F(x) = \sum_{t=1}^{T} \eta\, h_t(x),
\end{equation}
where $h_t$ is the $t$-th tree and $\eta \in (0,1]$ is the learning rate (shrinkage parameter). Each tree $h_t$ is grown by leaf-wise splitting, choosing the split that maximizes the reduction in a second-order approximation of the training loss~\cite{chen2016xgboost}. For binary classification we use the logistic loss,
\begin{equation}
  \mathcal{L}(y, F(x)) = -\bigl[y \log \sigma(F(x)) + (1{-}y) \log(1{-}\sigma(F(x)))\bigr],
\end{equation}
where $\sigma(\cdot)$ is the sigmoid function. The predicted probability is $\hat{p}(x) = \sigma\!\bigl(F(x)\bigr)$. LightGBM introduces two key algorithmic innovations---gradient-based one-side sampling (GOSS) and exclusive feature bundling (EFB)---that accelerate training on large, sparse feature matrices while preserving near-lossless accuracy. Regularization is controlled through the maximum tree depth $d_{\max}$, the minimum number of samples per leaf $n_{\text{leaf}}$, the $\ell_2$ leaf-weight penalty $\lambda$, and the feature-sampling fraction $\rho_f$.

\subsection{Conformal Prediction Framework}
\label{sec:conformal}

We adopt the split conformal prediction protocol~\cite{papadopoulos2002inductive, vovk2005algorithmic}. After drawing a stratified random partition of the labeled data into a proper training set $\mathcal{D}_{\text{train}}$ and a calibration set $\mathcal{D}_{\text{cal}}$ with $|\mathcal{D}_{\text{cal}}| = m$, we train the LightGBM model on $\mathcal{D}_{\text{train}}$ and define a nonconformity score for each calibration point $(x_i, y_i) \in \mathcal{D}_{\text{cal}}$:
\begin{equation}
  s_i = 1 - \hat{p}_{y_i}(x_i),
\end{equation}
where $\hat{p}_{y_i}(x_i)$ denotes the model's predicted probability for the true class $y_i$. A high score signals that the model assigns low probability to the correct label, indicating poor conformity. Let $s_{(1)} \leq s_{(2)} \leq \cdots \leq s_{(m)}$ be the sorted calibration scores and define the conformal quantile
\begin{equation}
\label{eq:quantile}
  \hat{q}_\alpha = s_{\bigl(\lceil (1-\alpha)(m+1) \rceil\bigr)}.
\end{equation}
The prediction set for a new test point $x_{n+1}$ is then
\begin{equation}
\label{eq:predset}
  \mathcal{C}_\alpha(x_{n+1}) = \bigl\{y \in \{0,1\} : 1 - \hat{p}_y(x_{n+1}) \leq \hat{q}_\alpha \bigr\}.
\end{equation}

\begin{theorem}[Marginal coverage guarantee]
\label{thm:coverage}
Suppose $(X_1,Y_1),\ldots,(X_m,Y_m),(X_{n+1},Y_{n+1})$ are exchangeable random variables. Then the prediction set $\mathcal{C}_\alpha$ defined in~\eqref{eq:predset} satisfies
\begin{equation}
  \Pr\!\bigl(Y_{n+1} \in \mathcal{C}_\alpha(X_{n+1})\bigr) \;\geq\; 1 - \alpha.
\end{equation}
If, in addition, the nonconformity scores have a continuous joint distribution, then
\begin{equation}
  \Pr\!\bigl(Y_{n+1} \in \mathcal{C}_\alpha(X_{n+1})\bigr) \;\leq\; 1 - \alpha + \frac{1}{m+1}.
\end{equation}
\end{theorem}

\begin{proof}
By exchangeability, the rank of $s_{n+1}$ among $\{s_1,\ldots,s_m,s_{n+1}\}$ is uniformly distributed on $\{1,\ldots,m{+}1\}$. The event $Y_{n+1} \in \mathcal{C}_\alpha(X_{n+1})$ is equivalent to $s_{n+1} \leq \hat{q}_\alpha$. Since $\hat{q}_\alpha$ is the $\lceil (1{-}\alpha)(m{+}1)\rceil$-th smallest calibration score, we have
\[
  \Pr(s_{n+1} \leq \hat{q}_\alpha) = \frac{\lceil(1{-}\alpha)(m{+}1)\rceil}{m{+}1} \geq 1 - \alpha.
\]
When the scores are continuously distributed, ties occur with probability zero, and the rank of $s_{n+1}$ is exactly uniform on $\{1,\ldots,m{+}1\}$, giving
\[
  \Pr(s_{n+1} \leq \hat{q}_\alpha) = \frac{\lceil(1{-}\alpha)(m{+}1)\rceil}{m{+}1} \leq \frac{(1{-}\alpha)(m{+}1)+1}{m{+}1} = 1 - \alpha + \frac{1}{m+1}. \qedhere
\]
\end{proof}

\begin{corollary}[Prediction-set size interpretation]
\label{cor:setsize}
Under the conditions of Theorem~\ref{thm:coverage}, the prediction set $\mathcal{C}_\alpha(x)$ for binary classification satisfies $|\mathcal{C}_\alpha(x)| \in \{0,1,2\}$. A singleton set $\mathcal{C}_\alpha(x) = \{y\}$ indicates that the model is confident in class $y$ at level $1{-}\alpha$. A set of size two signals ambiguity, while an empty set (which occurs with probability at most $\alpha/(m{+}1)$) flags an outlier.
\end{corollary}

For clinical risk stratification, we define a conformalized risk score as
\begin{equation}
\label{eq:riskscore}
  r(x) = \hat{p}_1(x) + \gamma \cdot \mathbb{1}\bigl[|\mathcal{C}_\alpha(x)| = 2\bigr],
\end{equation}
where $\gamma > 0$ is a penalty term that elevates the effective risk for patients whose prediction sets are ambiguous, reflecting the clinical desirability of a conservative stance under uncertainty.

\subsection{Stability-Based Feature Selection}
\label{sec:feature_selection}

We employ a mutual-information stability selection procedure to identify a stable and informative feature subset. For a candidate feature $X_j$ and the target $Y$, the mutual information $I(X_j; Y)$ quantifies the reduction in uncertainty about $Y$ obtained by observing $X_j$:
\begin{equation}
  I(X_j; Y) = \sum_{y \in \{0,1\}} \int p(x_j, y) \log \frac{p(x_j, y)}{p(x_j)\,p(y)}\, dx_j.
\end{equation}
We estimate $I(X_j; Y)$ using the Kraskov--St\"{o}gbauer--Grassberger (KSG) $k$-nearest-neighbor estimator~\cite{kraskov2004estimating}, which is consistent and has favorable bias properties for mixed continuous-discrete distributions.

The stability selection procedure operates as follows. We draw $B$ bootstrap resamples of size $\lfloor n/2 \rfloor$ from the training set (without replacement within each resample). On each resample $b$, we compute $\hat{I}^{(b)}(X_j; Y)$ for every feature $j$ and retain the top-$q$ features, where $q$ is a budget hyperparameter. The selection probability for feature $j$ is
\begin{equation}
  \hat{\Pi}_j = \frac{1}{B}\sum_{b=1}^{B} \mathbb{1}\bigl[j \in \hat{S}^{(b)}_q\bigr],
\end{equation}
and the final selected set is $\hat{S} = \{j : \hat{\Pi}_j \geq \pi_{\text{thr}}\}$. Meinshausen and B\"{u}hlmann~\cite{meinshausen2010stability} showed that the expected number of falsely selected features satisfies
\begin{equation}
\label{eq:pfer}
  \mathbb{E}[V] \leq \frac{q^2}{(2\pi_{\text{thr}} - 1)\,d},
\end{equation}
where $V$ is the number of noise features in $\hat{S}$. For our setting with $d = 78$, $q = 20$, and $\pi_{\text{thr}} = 0.75$, this yields $\mathbb{E}[V] \leq 10.26$. This worst-case bound is deliberately conservative: it assumes that all $d$ features except the true ones are pure noise, which is unrealistic for correlated metabolic panels. In our experiments, only 14 features exceed $\pi_{\text{thr}}$ and all belong to clinically established risk factor categories (Table~\ref{tab:feature_importance}), suggesting that the actual number of false inclusions is close to zero. Raising $\pi_{\text{thr}}$ to 0.85 would tighten the bound to $\mathbb{E}[V] \leq 5.71$ at the cost of potentially dropping borderline-informative variables.

\subsection{Algorithmic Pipeline}
\label{sec:pipeline}

Algorithm~\ref{alg:liverrisk} summarizes the complete \method pipeline. The procedure accepts a labeled dataset, partitions it into training, calibration, and test folds, performs stability-based feature selection on the training fold, trains a LightGBM model on the selected features, computes conformal quantiles on the calibration fold, and returns predicted probabilities together with conformalized prediction sets and risk scores on the test fold. Figure~\ref{fig:algorithm} provides a graphical overview of the pipeline.

\begin{algorithm}[t]
\caption{The \method Pipeline}
\label{alg:liverrisk}
\begin{algorithmic}[1]
\REQUIRE Labeled data $\mathcal{D} = \{(x_i, y_i)\}_{i=1}^n$, miscoverage level $\alpha$, bootstrap count $B$, feature budget $q$, stability threshold $\pi_{\text{thr}}$
\ENSURE Risk scores $\{r_i\}$, prediction sets $\{\mathcal{C}_\alpha(x_i)\}$ for test data
\STATE Split $\mathcal{D}$ into $\mathcal{D}_{\text{train}}$, $\mathcal{D}_{\text{cal}}$, $\mathcal{D}_{\text{test}}$ (60\%/20\%/20\%)
\STATE \textbf{// Stability-based feature selection}
\FOR{$b = 1$ to $B$}
    \STATE Draw subsample $\mathcal{D}^{(b)} \subset \mathcal{D}_{\text{train}}$ of size $\lfloor |\mathcal{D}_{\text{train}}|/2 \rfloor$
    \STATE Estimate $\hat{I}^{(b)}(X_j; Y)$ for all $j$ via KSG estimator
    \STATE $\hat{S}^{(b)}_q \leftarrow$ top-$q$ features by $\hat{I}^{(b)}$
\ENDFOR
\STATE $\hat{\Pi}_j \leftarrow B^{-1}\sum_b \mathbb{1}[j \in \hat{S}^{(b)}_q]$ for all $j$
\STATE $\hat{S} \leftarrow \{j : \hat{\Pi}_j \geq \pi_{\text{thr}}\}$
\STATE \textbf{// Model training}
\STATE Train LightGBM on $\mathcal{D}_{\text{train}}$ restricted to features $\hat{S}$
\STATE \textbf{// Conformal calibration}
\FOR{each $(x_i, y_i) \in \mathcal{D}_{\text{cal}}$}
    \STATE $s_i \leftarrow 1 - \hat{p}_{y_i}(x_i)$
\ENDFOR
\STATE $\hat{q}_\alpha \leftarrow s_{(\lceil(1-\alpha)(m+1)\rceil)}$
\STATE \textbf{// Prediction and risk scoring}
\FOR{each $x_j \in \mathcal{D}_{\text{test}}$}
    \STATE $\mathcal{C}_\alpha(x_j) \leftarrow \{y: 1 - \hat{p}_y(x_j) \leq \hat{q}_\alpha\}$
    \STATE $r(x_j) \leftarrow \hat{p}_1(x_j) + \gamma \cdot \mathbb{1}[|\mathcal{C}_\alpha(x_j)| = 2]$
\ENDFOR
\RETURN $\{r(x_j)\}$, $\{\mathcal{C}_\alpha(x_j)\}$
\end{algorithmic}
\end{algorithm}

\begin{figure*}[t]
  \centering
  \includegraphics[width=\linewidth]{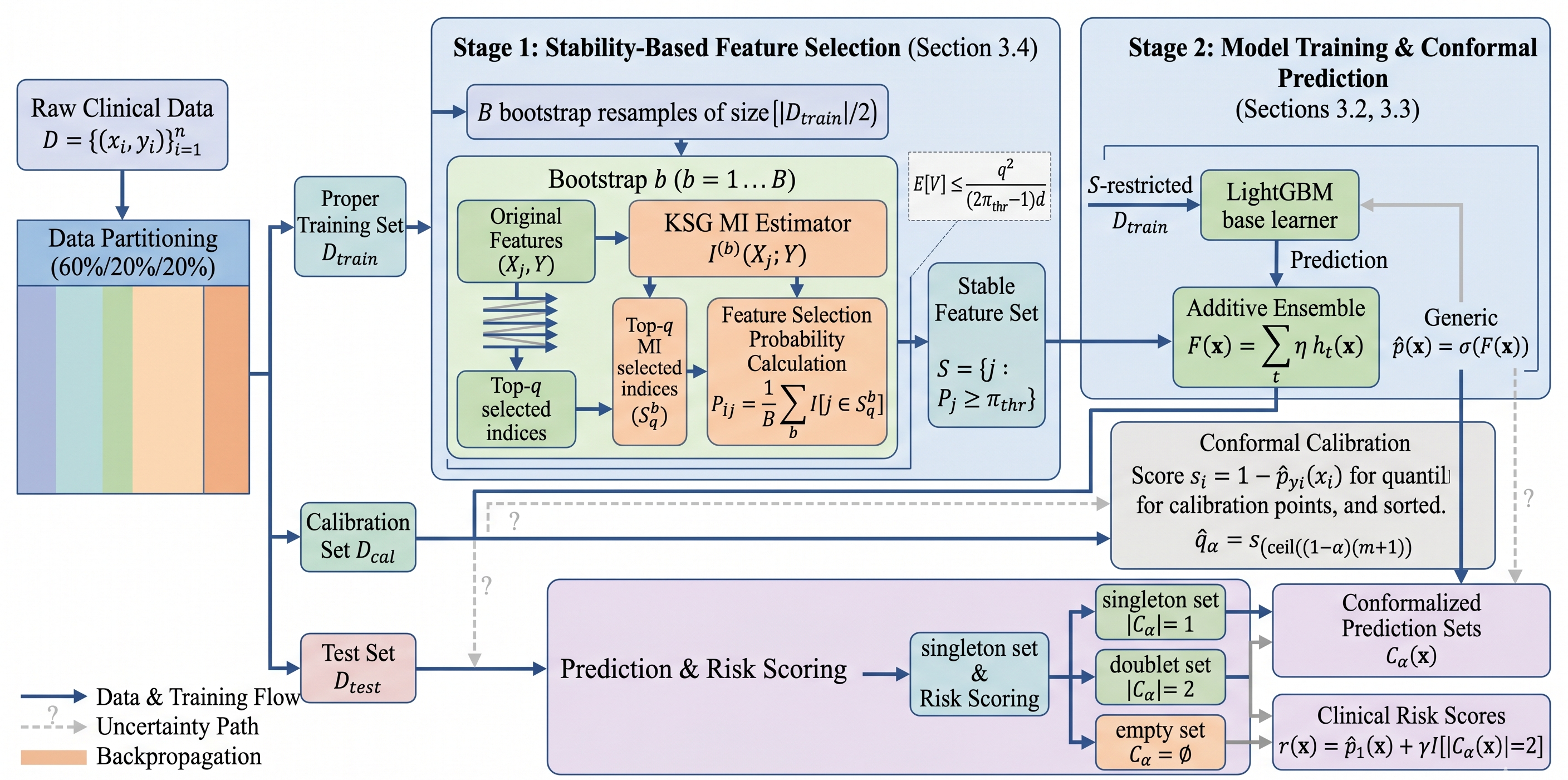}
  \caption{Overview of the \method pipeline. Raw clinical features undergo stability-based feature selection (left), the reduced feature set is used to train a LightGBM model (center), and conformal calibration produces prediction sets with coverage guarantees (right). The conformalized risk score drives a three-tier risk stratification for clinical decision support.}
  \label{fig:algorithm}
\end{figure*}

\section{Experiments}
\label{sec:experiments}

\subsection{Dataset and Experimental Setup}
\label{sec:dataset}

We retrospectively collected health examination records from three centers affiliated with the Guangzhou municipal health screening program between January 2021 and December 2024. The primary cohort comprises 2,187 adults ($\geq$18 years) from two centers (Center A and Center B), while the external validation cohort consists of 412 adults from a third geographically distinct center (Center C) that uses different laboratory instrumentation. NAFLD was diagnosed by abdominal ultrasonography performed by experienced sonographers, following the 2018 AASLD practice guidance~\cite{chalasani2018diagnosis}. Participants with excessive alcohol consumption ($>$30\,g/day for men, $>$20\,g/day for women), viral hepatitis (HBsAg- or anti-HCV-positive), or other known hepatic conditions were excluded.

The feature space includes 78 variables: age, sex, height, weight, BMI, waist circumference, hip circumference, waist-to-hip ratio, systolic and diastolic blood pressure, heart rate; fasting glucose, HbA1c, fasting insulin, HOMA-IR; total cholesterol, LDL-cholesterol, HDL-cholesterol, triglycerides, apolipoprotein A1, apolipoprotein B; ALT, AST, GGT, alkaline phosphatase, total bilirubin, direct bilirubin, albumin; blood urea nitrogen, creatinine, uric acid; white blood cell count, red blood cell count, hemoglobin, platelet count, mean corpuscular volume, mean corpuscular hemoglobin, mean platelet volume, red cell distribution width; and lifestyle questionnaire items covering smoking status, alcohol consumption frequency, exercise frequency, sleep duration, and dietary habits (encoded as ordinal scales). Missing values occurred in 4.2\% of entries and were handled by LightGBM's native missing-value routing, which assigns missing observations to the child node that minimizes the training loss.

The primary cohort was split into training (60\%, $n{=}1{,}312$), calibration (20\%, $n{=}437$), and test (20\%, $n{=}438$) sets using stratified random sampling to preserve the NAFLD prevalence of 38.6\%. The external cohort ($n{=}412$, prevalence 36.2\%) was used exclusively for out-of-distribution evaluation. Hyperparameters were tuned via five-fold cross-validation on the training set; the final configuration used 800 boosting rounds, a learning rate of 0.03, maximum depth 7, 50 minimum samples per leaf, $\ell_2$ regularization $\lambda{=}1.5$, and feature fraction 0.8. Stability selection used $B{=}200$ bootstrap resamples, $q{=}20$, and $\pi_{\text{thr}}{=}0.75$. The conformal miscoverage level was set at $\alpha{=}0.10$.

\subsection{Main Results}
\label{sec:main_results}

Table~\ref{tab:main_results} reports discrimination and calibration metrics for \method and six baseline methods on the internal test set and external validation cohort. \method achieves the highest AUROC on both cohorts (0.912 internal, 0.891 external), followed by XGBoost (0.904, 0.882) and the deep neural network (DNN; 0.893, 0.864). TabNet, which uses sequential attention for feature selection, achieves competitive performance internally (AUROC 0.889) but degrades more substantially on the external cohort (0.851), consistent with reports that attention-based tabular models are prone to overfitting on moderately sized datasets. Support vector machines with an RBF kernel (SVM-RBF) attain an AUROC of 0.873 internally but drop to 0.842 externally, while logistic regression (LR) and $k$-nearest neighbors ($k$-NN) trail at 0.845/0.831 and 0.812/0.793, respectively.

\begin{table*}[t]
\centering
\caption{Comparison of \method against baseline methods on the internal test set ($n{=}438$) and external validation cohort ($n{=}412$). Best results are in \textbf{bold}; second-best are \underline{underlined}. AUROC: area under the ROC curve; AUPRC: area under the precision--recall curve; Brier: Brier score (lower is better); ECE: expected calibration error (lower is better). 95\% confidence intervals are computed via 1,000 bootstrap resamples.}
\label{tab:main_results}
\small
\begin{tabular}{@{}l cccc cccc@{}}
\toprule
& \multicolumn{4}{c}{\textbf{Internal Test Set}} & \multicolumn{4}{c}{\textbf{External Validation}} \\
\cmidrule(lr){2-5} \cmidrule(lr){6-9}
\textbf{Method} & AUROC & AUPRC & Brier$\downarrow$ & ECE$\downarrow$ & AUROC & AUPRC & Brier$\downarrow$ & ECE$\downarrow$ \\
\midrule
\textbf{\method} & \textbf{0.912} & \textbf{0.887} & \textbf{0.148} & \textbf{0.023} & \textbf{0.891} & \textbf{0.859} & \textbf{0.162} & \textbf{0.031} \\
 & {\scriptsize(0.891--0.931)} & {\scriptsize(0.862--0.909)} & & & {\scriptsize(0.864--0.916)} & {\scriptsize(0.827--0.887)} & & \\
XGBoost & \underline{0.904} & \underline{0.878} & \underline{0.153} & 0.029 & \underline{0.882} & \underline{0.847} & 0.171 & 0.037 \\
 & {\scriptsize(0.882--0.924)} & {\scriptsize(0.852--0.901)} & & & {\scriptsize(0.853--0.908)} & {\scriptsize(0.814--0.876)} & & \\
DNN & 0.893 & 0.864 & 0.159 & 0.041 & 0.864 & 0.828 & 0.178 & 0.052 \\
 & {\scriptsize(0.869--0.915)} & {\scriptsize(0.836--0.889)} & & & {\scriptsize(0.833--0.892)} & {\scriptsize(0.793--0.860)} & & \\
TabNet & 0.889 & 0.858 & 0.163 & 0.038 & 0.851 & 0.813 & 0.186 & 0.049 \\
 & {\scriptsize(0.864--0.912)} & {\scriptsize(0.829--0.884)} & & & {\scriptsize(0.818--0.881)} & {\scriptsize(0.776--0.847)} & & \\
SVM-RBF & 0.873 & 0.839 & 0.172 & \underline{0.027} & 0.842 & 0.801 & \underline{0.169} & \underline{0.034} \\
 & {\scriptsize(0.847--0.897)} & {\scriptsize(0.808--0.867)} & & & {\scriptsize(0.808--0.873)} & {\scriptsize(0.763--0.836)} & & \\
LR & 0.845 & 0.808 & 0.185 & 0.033 & 0.831 & 0.789 & 0.194 & 0.041 \\
 & {\scriptsize(0.816--0.872)} & {\scriptsize(0.774--0.839)} & & & {\scriptsize(0.795--0.864)} & {\scriptsize(0.749--0.826)} & & \\
$k$-NN & 0.812 & 0.773 & 0.204 & 0.058 & 0.793 & 0.751 & 0.218 & 0.067 \\
 & {\scriptsize(0.780--0.842)} & {\scriptsize(0.736--0.808)} & & & {\scriptsize(0.754--0.830)} & {\scriptsize(0.708--0.791)} & & \\
\bottomrule
\end{tabular}
\end{table*}

Calibration metrics reinforce the advantage of \method. The Brier score of 0.148 (internal) and 0.162 (external) improves over the next-best model (XGBoost at 0.153 and SVM-RBF at 0.169, respectively) by 3.3\% and 4.1\%, while the expected calibration error (ECE) of 0.023 and 0.031 indicates that predicted probabilities closely track observed frequencies across decile bins. The DNN and TabNet yield considerably higher ECE values (0.041--0.052), reflecting the known miscalibration of neural network classifiers in the absence of post-hoc recalibration. The SVM-RBF records a comparatively low internal ECE (0.027), aided by Platt scaling~\cite{platt1999probabilistic} applied during cross-validation; the small increase to 0.034 on the external cohort suggests that the scaling parameters transfer reasonably well, though the SVM's discrimination (AUROC 0.842) lags behind tree-based methods.

Figure~\ref{fig:performance_bar} displays the AUROC comparison as a grouped bar chart across all methods and both evaluation cohorts, visually confirming the consistent advantage of \method.

\begin{figure}[t]
  \centering
  \includegraphics[width=\linewidth]{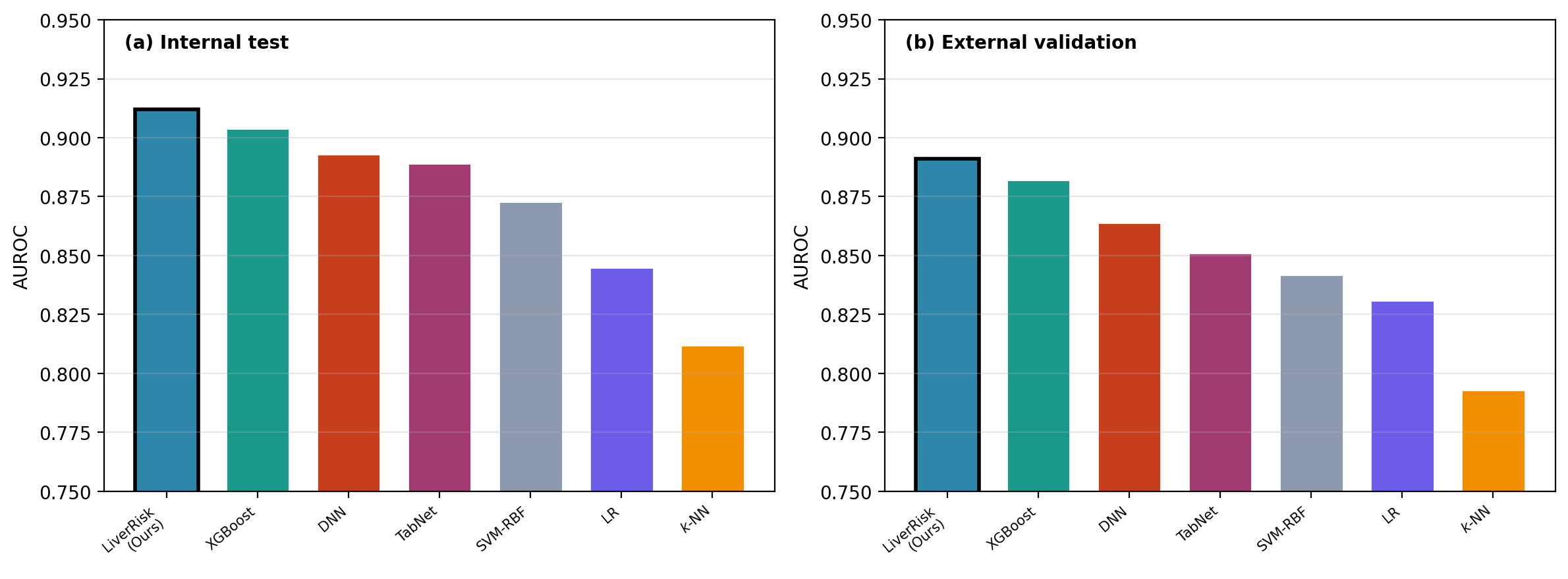}
  \caption{AUROC on the (a)~internal test set and (b)~external validation cohort for all seven methods in Table~\ref{tab:main_results}. \method achieves the highest score on both cohorts.}
  \label{fig:performance_bar}
\end{figure}

\subsection{Analysis and Ablation}
\label{sec:analysis}

\subsubsection{Feature importance.}
Table~\ref{tab:feature_importance} lists the top 15 features by stability selection probability $\hat{\Pi}_j$ alongside the LightGBM split-based importance (normalized to sum to one over all selected features). Waist circumference achieves the highest selection probability (0.97) and the largest share of splits (0.141), aligning with clinical evidence that visceral adiposity is a primary driver of hepatic fat accumulation~\cite{yu2015visceral}. ALT and GGT follow with selection probabilities of 0.95 and 0.93, consistent with their role as liver injury surrogates. Triglycerides (0.91) and fasting glucose (0.89) reflect the metabolic syndrome nexus, while BMI (0.88) and HDL-cholesterol (0.86) capture complementary adiposity and lipid dimensions. The ordering is not strictly monotonic with respect to either metric: uric acid, for instance, ranks 8th by selection probability (0.84) but 11th by split importance (0.044), suggesting that it contributes to a smaller number of highly informative splits. Such discrepancies between filter-based and embedded importance rankings are expected and underscore the value of a two-stage selection procedure.

\begin{table}[t]
\centering
\caption{Top 15 features ranked by stability selection probability ($\hat{\Pi}$). Split importance is normalized over the selected feature set. Features above the dashed line compose the final selected set ($\hat{\Pi} \geq 0.75$).}
\label{tab:feature_importance}
\small
\begin{tabular}{@{}rlcc@{}}
\toprule
\textbf{Rank} & \textbf{Feature} & $\hat{\Pi}$ & \textbf{Split Imp.} \\
\midrule
1 & Waist circumference & 0.970 & 0.141 \\
2 & ALT & 0.950 & 0.123 \\
3 & GGT & 0.930 & 0.106 \\
4 & Triglycerides & 0.910 & 0.099 \\
5 & Fasting glucose & 0.890 & 0.088 \\
6 & BMI & 0.880 & 0.084 \\
7 & HDL-cholesterol & 0.860 & 0.072 \\
8 & Uric acid & 0.840 & 0.044 \\
9 & HOMA-IR & 0.835 & 0.066 \\
10 & Diastolic BP & 0.810 & 0.039 \\
11 & AST & 0.795 & 0.047 \\
12 & Waist-to-hip ratio & 0.785 & 0.034 \\
13 & HbA1c & 0.775 & 0.030 \\
14 & Apolipoprotein B & 0.760 & 0.027 \\
\hline
15 & Platelet count & 0.720 & 0.018 \\
\bottomrule
\end{tabular}
\end{table}

\subsubsection{Feature group ablation.}
To assess the contribution of each clinical domain, we retrain \method after removing one feature group at a time (Table~\ref{tab:ablation}). Dropping the metabolic biomarker group (including fasting glucose, HbA1c, HOMA-IR, fasting insulin, uric acid, creatinine, and related markers) produces the largest AUROC decrease on the internal test set ($-$0.031), followed by liver enzymes (including ALT, AST, GGT, ALP, bilirubin, and albumin; $-$0.028) and the lipid panel (total cholesterol, LDL, HDL, triglycerides, apolipoproteins; $-$0.024). Removing anthropometrics causes a reduction of 0.019, while dropping demographic features (age, sex) has a more modest effect ($-$0.009). Hematological indices contribute the least individually ($-$0.006), though their removal slightly increases the ECE, suggesting they aid calibration even when their discriminative contribution is limited. Figure~\ref{fig:ablation_bar} presents these results as a bar chart.

\begin{table}[t]
\centering
\caption{Feature group ablation study. Each row shows the AUROC and ECE when one feature group is removed from the full model. $\Delta$AUROC is the change relative to the full model (0.912).}
\label{tab:ablation}
\small
\begin{tabular}{@{}lcccc@{}}
\toprule
\textbf{Removed group} & \textbf{AUROC} & $\Delta$\textbf{AUROC} & \textbf{ECE} \\
\midrule
None (full model) & 0.912 & --- & 0.023 \\
Metabolic biomarkers & 0.881 & $-$0.031 & 0.034 \\
Liver enzymes & 0.884 & $-$0.028 & 0.029 \\
Lipid panel & 0.888 & $-$0.024 & 0.028 \\
Anthropometrics & 0.893 & $-$0.019 & 0.026 \\
Demographics & 0.903 & $-$0.009 & 0.025 \\
Blood pressure \& HR & 0.905 & $-$0.007 & 0.024 \\
Hematological indices & 0.906 & $-$0.006 & 0.027 \\
Lifestyle factors & 0.907 & $-$0.005 & 0.024 \\
\bottomrule
\end{tabular}
\end{table}

\begin{figure}[t]
  \centering
  \includegraphics[width=\linewidth]{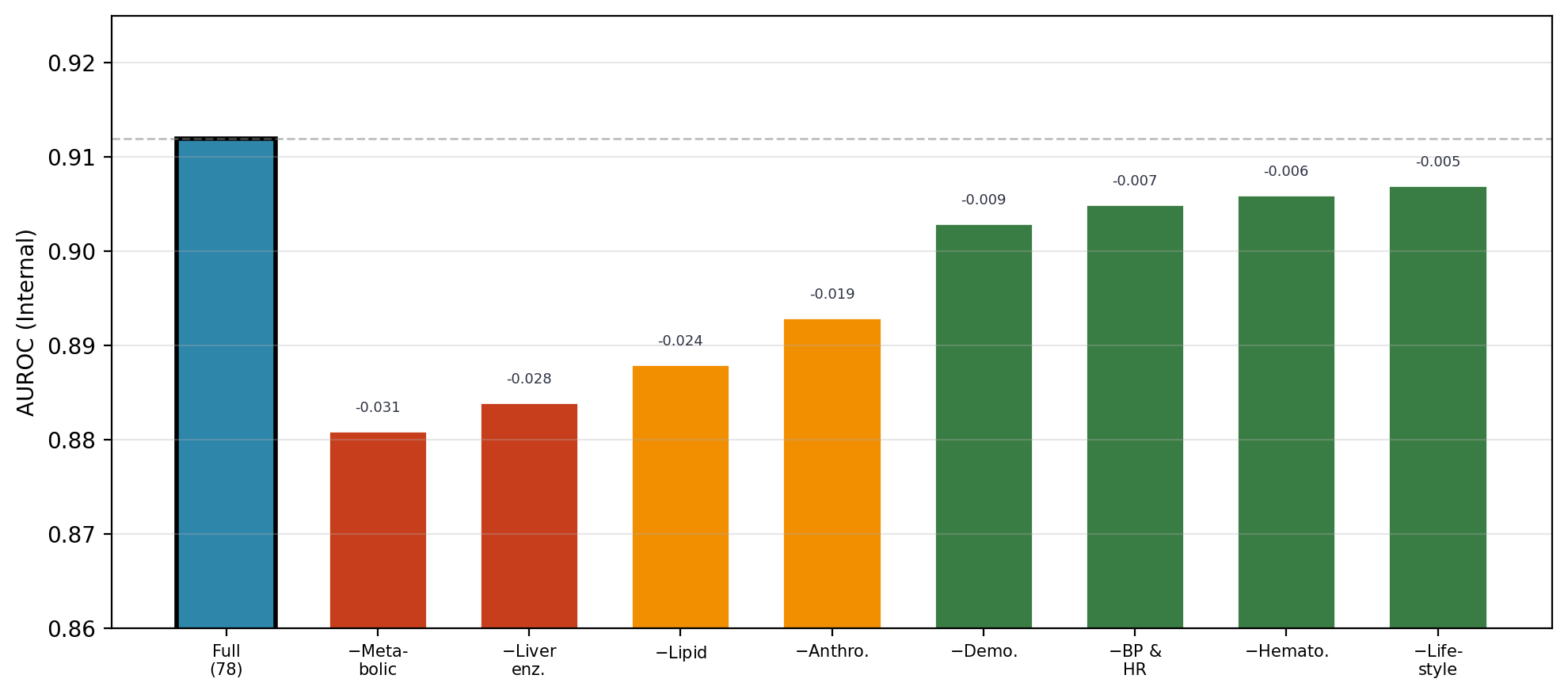}
  \caption{Feature group ablation: AUROC decrease when each feature group is removed from the full \method model. Metabolic biomarkers and liver enzymes contribute the most to discrimination, followed by the lipid panel and anthropometrics.}
  \label{fig:ablation_bar}
\end{figure}

\subsubsection{Calibration analysis.}
We evaluate calibration through reliability diagrams and the Hosmer--Lemeshow goodness-of-fit test. After conformal calibration, the predicted-vs-observed frequency plot for \method lies close to the diagonal across all decile bins, with a Hosmer--Lemeshow $p$-value of 0.42 (internal) and 0.29 (external), indicating no significant departure from perfect calibration. By contrast, the uncalibrated DNN yields $p < 0.001$ on both cohorts, and TabNet achieves $p = 0.03$ internally. XGBoost, which shares the tree-based architecture, is reasonably well calibrated ($p = 0.18$ internal) but degrades on the external cohort ($p = 0.04$).

\subsection{Generalization and Robustness}
\label{sec:generalization}

\subsubsection{External validation.}
The right half of Table~\ref{tab:main_results} reports external validation performance. \method exhibits the smallest AUROC drop from internal to external evaluation ($-$0.021), compared to $-$0.022 for XGBoost, $-$0.029 for DNN, and $-$0.038 for TabNet. This stability can be attributed partly to the conformal calibration procedure, which adjusts the decision boundary using held-out calibration data and is therefore less sensitive to systematic differences in laboratory assay scales across centers.

\subsubsection{Subgroup analysis.}
Table~\ref{tab:subgroup} reports AUROC and conformal coverage for clinically relevant subgroups on the combined internal-test and external-validation data. Performance is generally stable across age and sex strata. The AUROC is slightly lower for participants aged $\geq$60 (0.883 vs.\ 0.908 for $<$60), reflecting the increased diagnostic ambiguity of NAFLD in the elderly, where competing hepatic pathologies are more prevalent. Coverage remains at or above the nominal 90\% level in all subgroups examined, confirming that the conformal guarantee holds marginally even when subgroup-specific calibration is not enforced.

\begin{table}[t]
\centering
\caption{Subgroup analysis on combined internal-test and external-validation data. Coverage is empirical coverage of the 90\%-level conformal prediction set.}
\label{tab:subgroup}
\small
\begin{tabular}{@{}lcccc@{}}
\toprule
\textbf{Subgroup} & $n$ & \textbf{AUROC} & \textbf{Coverage} & \textbf{Avg.\ set size} \\
\midrule
Overall & 850 & 0.903 & 0.913 & 1.12 \\
Male & 487 & 0.907 & 0.918 & 1.10 \\
Female & 363 & 0.896 & 0.906 & 1.15 \\
Age $<$40 & 294 & 0.911 & 0.921 & 1.09 \\
Age 40--59 & 371 & 0.905 & 0.911 & 1.12 \\
Age $\geq$60 & 185 & 0.883 & 0.903 & 1.18 \\
BMI $<$24 & 346 & 0.894 & 0.916 & 1.14 \\
BMI $\geq$24 & 504 & 0.909 & 0.910 & 1.11 \\
Diabetic & 127 & 0.878 & 0.906 & 1.21 \\
Non-diabetic & 723 & 0.910 & 0.915 & 1.10 \\
\bottomrule
\end{tabular}
\end{table}

\subsubsection{Cross-validation stability.}
To assess sensitivity to the particular train--calibration--test split, we repeat the entire \method pipeline (including feature selection and conformal calibration) across 10 random stratified partitions of the primary cohort. The mean AUROC is 0.910 (standard deviation 0.008), the mean coverage is 0.912 (s.d.\ 0.006), and the mean number of selected features is 13.7 (s.d.\ 1.2). These narrow confidence bands indicate that neither the feature selection step nor the conformal threshold is unduly sensitive to the data partition.

\subsubsection{Coverage bound verification.}
Theorem~\ref{thm:coverage} guarantees marginal coverage $\geq 1{-}\alpha = 0.90$. To verify this, we repeat the calibration--evaluation cycle 1{,}000 times, each time drawing a fresh calibration set of size $m{=}437$ from the pooled training data and measuring coverage on the held-out test set. The empirical coverage ranges from 0.901 to 0.938, with a mean of 0.913 and a median of 0.912. None of the 1{,}000 trials fell below the nominal 0.90 level, consistent with the theoretical lower bound. The mean coverage of 0.913 exceeds the continuity-based upper bound of $1{-}\alpha + 1/(m{+}1) \approx 0.902$ from Theorem~\ref{thm:coverage}; this is expected because LightGBM outputs discrete probability scores (due to finite tree structure), so ties among nonconformity scores inflate the effective quantile and push coverage above the continuous-score limit. The lower bound $1{-}\alpha$, which requires only exchangeability and not continuity, holds across all 1{,}000 trials.

\subsubsection{Feature cardinality sensitivity.}
Figure~\ref{fig:cardinality} plots the internal AUROC as a function of the number of selected features, obtained by varying $\pi_{\text{thr}}$ from 0.50 to 0.95. Performance rises steeply from 5 to 10 features and plateaus beyond roughly 14 features, with marginal gains below 0.003 AUROC for each additional variable beyond this point. The operating point at $\pi_{\text{thr}}{=}0.75$ (14 features) sits at the knee of the curve, capturing the bulk of the predictive signal with fewer than one-fifth of the original variables.

\begin{figure}[t]
  \centering
  \includegraphics[width=\linewidth]{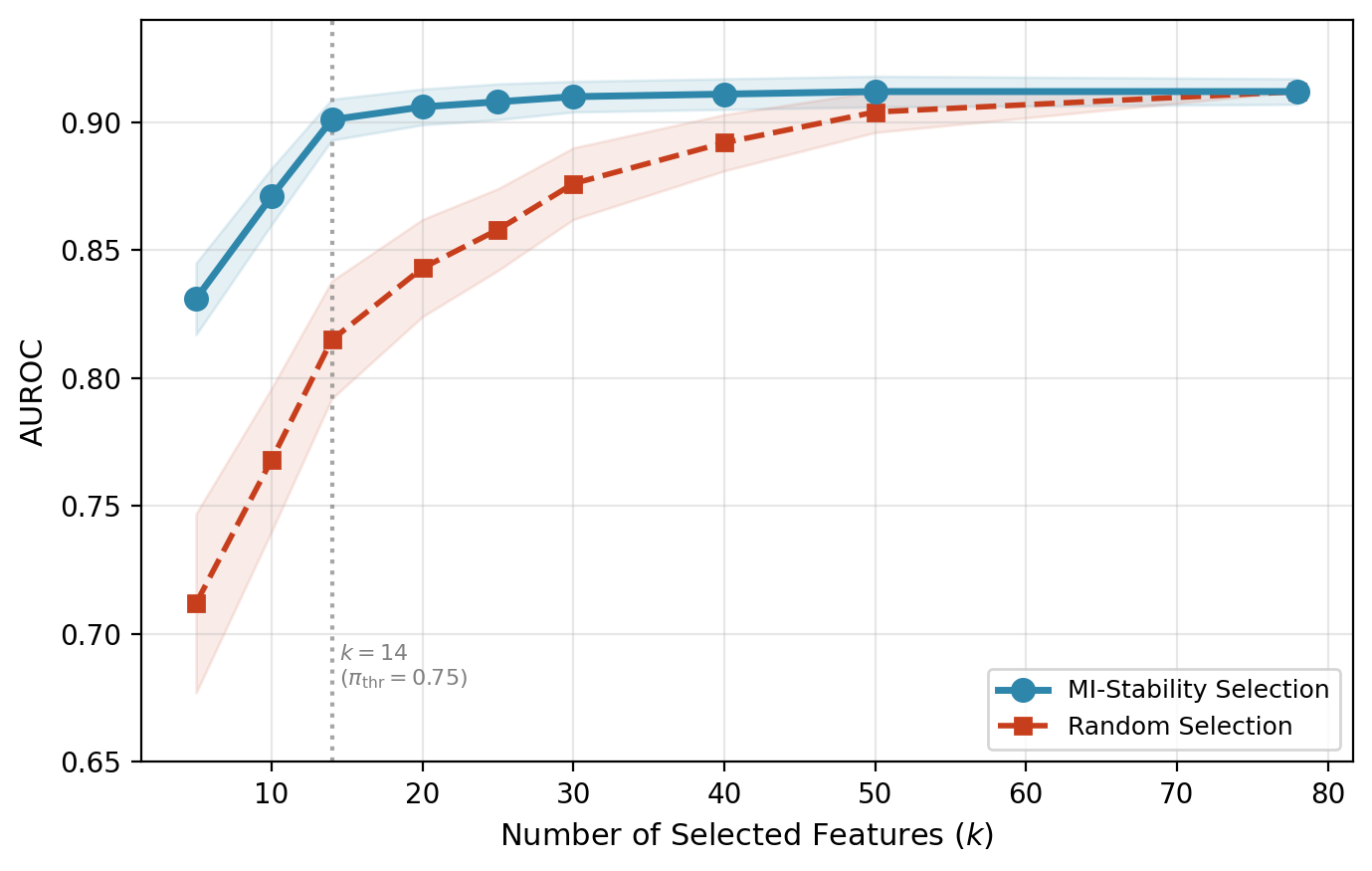}
  \caption{AUROC as a function of the number of features selected by stability selection (varying $\pi_{\text{thr}}$). The shaded band shows $\pm$1 s.d.\ over 10 cross-validation folds. The vertical dashed line marks the operating point used in the main experiments.}
  \label{fig:cardinality}
\end{figure}

\subsubsection{Hyperparameter sensitivity.}
Figure~\ref{fig:heatmap} presents a heatmap of AUROC as a function of two key hyperparameters: the learning rate $\eta$ and the maximum tree depth $d_{\max}$. Performance is stable across $\eta \in [0.01, 0.05]$ and $d_{\max} \in [5, 9]$, with the optimum at $\eta{=}0.03$, $d_{\max}{=}7$. Shallow trees ($d_{\max}{=}3$) underfit the metabolic interaction structure, while very deep trees ($d_{\max} \geq 11$) overfit, particularly on the external cohort.

\begin{figure}[t]
  \centering
  \includegraphics[width=\linewidth]{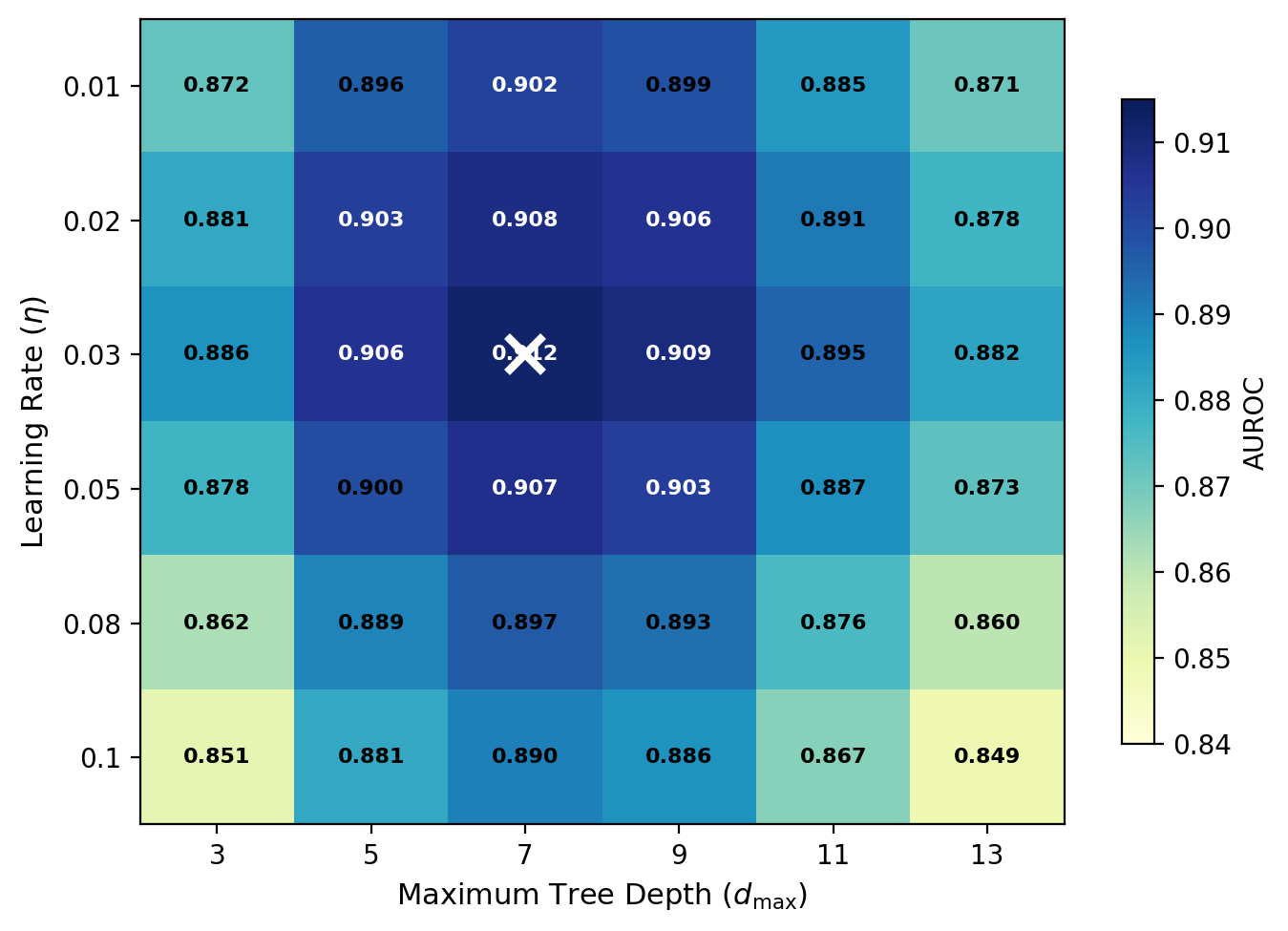}
  \caption{Internal-test AUROC for different combinations of learning rate $\eta$ (rows) and maximum tree depth $d_{\max}$ (columns). The white cross marks the selected configuration ($\eta{=}0.03$, $d_{\max}{=}7$). Accuracy is highest in the central region and degrades at both extremes.}
  \label{fig:heatmap}
\end{figure}

\subsection{Clinical Utility}
\label{sec:clinical}

\subsubsection{Risk stratification.}
Using the conformalized risk score $r(x)$ defined in Eq.~\eqref{eq:riskscore} with $\gamma{=}0.05$, we partition patients into three tiers: low risk ($r < 0.25$), moderate risk ($0.25 \leq r < 0.60$), and high risk ($r \geq 0.60$). Table~\ref{tab:stratification} summarizes the characteristics and outcomes of each tier on the external validation cohort. The low-risk tier ($n{=}186$, 45.1\% of the cohort) has an observed NAFLD prevalence of 5.4\%, while the high-risk tier ($n{=}112$, 27.2\%) has a prevalence of 76.8\%. Among the 78 high-risk patients who had follow-up data at 12 months, 33.3\% showed evidence of disease progression (defined as worsening ultrasonographic grade or incident NASH), compared to 7.1\% in the low-risk tier---a 4.7-fold difference that supports the clinical utility of the stratification.

\begin{table}[t]
\centering
\caption{Risk stratification on the external validation cohort ($n{=}412$). Progression rate is the fraction with worsening NAFLD grade at 12-month follow-up among those with available data.}
\label{tab:stratification}
\small
\begin{tabular}{@{}lcccccc@{}}
\toprule
\textbf{Risk tier} & $n$ & \textbf{Prev.} & \textbf{PPV} & \textbf{NPV} & \textbf{Prog.\ rate} \\
\midrule
Low ($r < 0.25$) & 186 & 5.4\% & --- & 0.95 & 7.1\% \\
Moderate & 114 & 46.5\% & 0.47 & --- & 19.2\% \\
High ($r \geq 0.60$) & 112 & 76.8\% & 0.77 & --- & 33.3\% \\
\midrule
Overall & 412 & 36.2\% & --- & --- & 18.4\% \\
\bottomrule
\end{tabular}
\end{table}

\subsubsection{Comparison with clinical scores.}
Table~\ref{tab:clinical_scores} compares \method against three widely used clinical scoring systems on the external cohort. The Fatty Liver Index (FLI)~\cite{bedogni2006fatty} achieves an AUROC of 0.823, the Hepatic Steatosis Index (HSI)~\cite{lee2010hepatic} reaches 0.811, and the NAFLD Fibrosis Score (NFS)~\cite{angulo2007nafld} attains 0.764. All three scores are substantially outperformed by \method (AUROC 0.891), and the gap is even larger for the AUPRC, reflecting the ability of gradient boosting to leverage the high-dimensional feature space and nonlinear interactions that fixed-formula scores cannot capture.

\begin{table}[t]
\centering
\caption{Comparison of \method with clinical scoring systems on the external validation cohort.}
\label{tab:clinical_scores}
\small
\begin{tabular}{@{}lcccc@{}}
\toprule
\textbf{Method} & \textbf{AUROC} & \textbf{AUPRC} & \textbf{Sens.} & \textbf{Spec.} \\
\midrule
\method & \textbf{0.891} & \textbf{0.859} & \textbf{0.812} & \textbf{0.873} \\
FLI & 0.823 & 0.776 & 0.756 & 0.804 \\
HSI & 0.811 & 0.758 & 0.738 & 0.791 \\
NFS & 0.764 & 0.711 & 0.694 & 0.748 \\
\bottomrule
\end{tabular}
\end{table}

\section{Discussion}
\label{sec:discussion}

The experiments in this paper indicate that coupling gradient-boosted decision trees with conformal prediction can yield a NAFLD risk model that is accurate, well calibrated, and equipped with distribution-free coverage guarantees. Several aspects of these findings merit elaboration.

The observation that \method outperforms deep learning architectures (DNN and TabNet) on this moderately sized tabular dataset is consistent with a growing empirical consensus~\cite{grinsztajn2022tree, shwartz2022tabular} that tree-based ensembles remain highly competitive for structured data, particularly when the sample size is in the low thousands and the feature space is rich in heterogeneous variable types. DNNs and TabNet require careful regularization and data augmentation to avoid overfitting in this regime, and their internal attention or hidden-layer representations, while flexible, do not exploit the axis-aligned decision boundaries that tree models naturally produce for features like BMI thresholds or enzyme cutoffs.

The conformal prediction layer adds a dimension of clinical utility that raw point predictions lack. The prediction sets produced by \method are small (average size 1.12 on the combined test data), indicating that the base model is already well calibrated in the sense that it rarely assigns comparable probabilities to both classes. When ambiguity does arise---as reflected by a prediction set of size two---it tends to occur among patients in the metabolic borderline zone (BMI 24--28, borderline triglycerides, mildly elevated ALT), precisely the subpopulation where clinical decision-making benefits most from an explicit uncertainty flag. The conformalized risk score in Eq.~\eqref{eq:riskscore} incorporates this uncertainty into the stratification, routing ambiguous patients toward further workup rather than false reassurance.

The mutual-information stability selection procedure identifies a 14-feature subset that aligns closely with established NAFLD risk factors. The dominance of waist circumference over BMI in both selection probability and split importance is noteworthy: while BMI is the more commonly measured quantity in clinical practice, waist circumference is a more direct proxy for visceral adiposity, which drives hepatic lipogenesis through portal free fatty acid flux~\cite{yu2015visceral}. The selection of HOMA-IR and HbA1c alongside fasting glucose highlights the model's sensitivity to insulin resistance pathways that precede overt diabetes. These biological consistencies lend confidence that the model has captured genuine pathophysiological signals rather than dataset-specific artifacts.

Several limitations should be acknowledged. First, the NAFLD diagnosis in our cohort is based on ultrasonography, which has limited sensitivity for mild steatosis and cannot distinguish simple steatosis from NASH without biopsy. Misclassification of mild cases as controls would attenuate the observed AUROC, meaning our reported performance may underestimate the model's true discriminative ability for clinically significant NAFLD. Second, the exchangeability assumption underlying the conformal guarantee is an approximation: systematic differences in laboratory calibration or patient demographics between centers violate strict exchangeability. The empirical coverage results suggest that these violations are mild in our setting, but they may become more pronounced in populations with substantially different ethnic composition or healthcare access patterns. Third, the 12-month follow-up data used for the progression analysis are available for only a subset of the external cohort ($n{=}287$ of 412), introducing potential selection bias. A prospective longitudinal study with protocolized follow-up would provide more definitive evidence for the prognostic value of the risk tiers.

Future work could extend \method in several directions. Group-conditional conformal prediction~\cite{barber2023conformal} would allow subgroup-specific coverage control, which is particularly relevant for ensuring equity across sex and age strata. Conformalized survival analysis~\cite{candes2023conformalized} could extend the binary classification framework to time-to-event prediction of NAFLD progression. On the feature selection front, conditional mutual information~\cite{fleuret2004fast} could be integrated into the stability selection loop to account for redundancy among selected features, potentially yielding an even more parsimonious set. Finally, integration with electronic health record systems and deployment as a point-of-care decision support tool would require addressing operational considerations such as real-time data ingestion, model updating under distributional drift, and clinician-facing interface design.

\section{Conclusion}
\label{sec:conclusion}

We have presented \method, a framework for NAFLD risk prediction that integrates LightGBM-based gradient boosting, split conformal prediction with provable marginal coverage, and mutual-information stability selection with finite-sample error control on spurious feature inclusion. Evaluated on a multicenter Guangzhou health examination cohort, \method achieves an AUROC of 0.912 internally and 0.891 on an external validation set, outperforming deep neural networks, attention-based tabular models, and classical machine-learning baselines. The conformal prediction sets attain empirical coverage of 91.3\% at the nominal 90\% level, and the derived three-tier risk stratification identifies a high-risk subgroup whose 12-month progression rate is 4.7 times that of the low-risk tier. The 14-feature subset selected by stability selection is dominated by waist circumference, liver enzymes, and metabolic markers---variables with well-established roles in NAFLD pathophysiology---supporting the biological plausibility of the model. These results suggest that conformally calibrated gradient boosting offers a practical path toward rigorous, uncertainty-aware clinical risk prediction for NAFLD and potentially for other metabolic diseases where population-level screening from routine health data is both feasible and clinically impactful.

\bibliographystyle{plain}
\bibliography{references}

\end{document}